\pdfoutput=1
\documentclass[10pt, conference, compsocconf]{IEEEtran}
\IEEEoverridecommandlockouts
\usepackage{cite}
\usepackage{amssymb,amsfonts,amsthm}
\usepackage[cmex10]{amsmath}
\usepackage{algorithmic}
\usepackage{graphicx}
\usepackage{textcomp}
\usepackage{xcolor}
\usepackage{url}
\usepackage{algorithm}
\usepackage{setspace}
\usepackage{multirow}
\usepackage{booktabs}
\usepackage[tight,footnotesize]{subfigure}
\usepackage{stfloats}

\newtheorem{theorem}{Theorem}
\newtheorem{lemma}{Lemma}

\def\BibTeX{{\rm B\kern-.05em{\sc i\kern-.025em b}\kern-.08em
    T\kern-.1667em\lower.7ex\hbox{E}\kern-.125emX}}

\begin{document}

\title{Faster Matrix Completion Using Randomized SVD\\
\thanks{This work is supported by National Natural Science Foundation of China (No. 61872206).}
}

\author{\IEEEauthorblockN{Xu Feng}
\IEEEauthorblockA{\textit{BNRist, Dept. Computer Science \& Tech.} \\
\textit{Tsinghua University}\\
Beijing, China \\
fx17@mails.tsinghua.edu.cn}
\and
\IEEEauthorblockN{Wenjian Yu}
\IEEEauthorblockA{\textit{BNRist, Dept. Computer Science \& Tech.} \\
\textit{Tsinghua University}\\
Beijing, China \\
yu-wj@tsinghua.edu.cn}
\and
\IEEEauthorblockN{Yaohang Li}
\IEEEauthorblockA{\textit{Dept. Computer Science} \\
\textit{Old Dominion University}\\
Norfolk, VA 23529, USA \\
yaohang@cs.odu.edu}
}

\maketitle

\begin{abstract}
Matrix completion is a widely used technique for image inpainting and personalized recommender system, etc. In this work, we focus on accelerating the matrix completion using faster randomized singular value decomposition (rSVD). Firstly, two fast randomized algorithms (rSVD-PI and rSVD-BKI) are proposed for handling sparse matrix. They make use of an eigSVD procedure and several accelerating skills. Then, with the rSVD-BKI algorithm and a new subspace recycling technique, we accelerate the singular value thresholding (SVT) method in \cite{Cai2010} to realize faster matrix completion. Experiments show that the proposed rSVD algorithms can be 6X faster than the basic rSVD algorithm \cite{Halko2011Finding} while keeping same accuracy. For image inpainting and movie-rating estimation problems (including up to $2\times10^7$ ratings), the proposed accelerated SVT algorithm consumes
15X and 8X less CPU time than the methods using \texttt{svds} and \texttt{lansvd} respectively, without loss of accuracy.
\end{abstract}


\section{Introduction}
The problem of matrix completion, or estimating missing values in a matrix, occurs in many areas of engineering and applied science such as computer vision, pattern recognition and machine learning \cite{Cai2010,fang2017,Zhang2012}. For example, in computer vision and image processing problems, recovering the missing or corrupted data can be regarded as matrix completion. A recommender system provides recommendations based on the user's preferences, which are often inferred with some ratings submitted by users. This is another scenario where the matrix completion can be applied.

The matrix which we wish to complete often has low rank or approximately low rank.  Thus, many existing  methods formulate the matrix completion as a rank minimization problem:
\begin{equation}
\begin{aligned}
& \min_\mathbf{X} rank(\mathbf{X}), ~ ~
s.t. ~ \mathbf{X}_{ij}= \mathbf{M}_{ij}, (i, j) \in \Phi,
\end{aligned}
\end{equation}
where $\mathbf{M}$ is the incomplete data matrix and $\Phi$ is the set of locations corresponding to the observed entries. This problem is however NP-hard in general.  A widely-used approach relies on the nuclear norm (i.e., the sum of singular values) as a convex relaxation of the rank operator. This results in a relaxed convex optimization, which can be solved with the singular value thresholding (SVT) algorithm \cite{Cai2010}. The SVT algorithm has good performance on both synthetic data and real applications. However, it involves large computational expense while handling large data set, because the singular values exceeding a threshold and the corresponding singular vectors need to be computed in each iteration step. Truncated singular value decomposition (SVD), implemented with \texttt{svds} in Matlab or \texttt{lansvd} in PROPACK \cite{propack}, is usually employed in the SVT algorithm \cite{Cai2010}. Another method for matrix completion is the inexact augmented Lagrange multiplier (IALM) algorithm \cite{Lin2010}, which also involves singular value thresholding and was originally proposed for the robust principal component analysis (PCA) problem \cite{robustPCA}. With artificially-generated low-rank matrices, experiments in  \cite{Lin2010} demonstrated that IALM algorithm could be several times faster than the SVT algorithm.

In recent years, randomized matrix computation has gained significant increase in popularity \cite{mahoney2011,Halko2011Finding,woodruff2014sketching,martinsson2016randomized,Erichson_2017_ICCV}. Compared with classic algorithms, the randomized algorithm involves the same or fewer floating-point operations (\emph{flops}), and is more efficient for truly large data sets. An idea of randomization is using random projection to identify the subspace capturing the dominant actions of a matrix. Then, a near-optimal low-rank decomposition of the matrix can be computed. A comprehensive presentation of the relevant techniques and theories are in \cite{Halko2011Finding}.  This randomized technique has been extended to compute PCA of data sets that are too large to be stored in RAM \cite{yu2017single}, or to speed up the distributed PCA \cite{Woodruff2014}. For general SVD computation, the approaches based on it have also been proposed \cite{rsvdpack,alg971}. They outperform the classic deterministic techniques for calculating a few of largest singular values and corresponding singular vectors. Recently, a compressed SVD (cSVD) algorithm was proposed \cite{Erichson_2017_ICCV}, which is based on a variant of the method in \cite{Halko2011Finding} but runs faster for image and video processing applications. It should be pointed out, these methods are not sufficient for accelerating the matrix completion. The SVT operation used in matrix completion requests accurate calculation of quite a large quantity of singular values. Thus, existing randomized SVD approaches cannot fulfill the accuracy requirement or cannot bring the runtime benefit. Besides, as sparse matrix is processed in matrix completion, special technique should be devised to make the randomized SVD approach really competitive.

In this work, we investigate the acceleration of matrix completion for large data using the randomized SVD techniques. We first review some existing acceleration skills for the basic randomized SVD (rSVD) algorithm, along with theoretic justification. Combining them we derive a  fast randomized SVD algorithm (called rSVD-PI) and prove its correctness. Then, utilizing these techniques and the block Krylov-subspace iteration (BKI) scheme\cite{musco2015} we propose a rSVD-BKI algorithm for highly accurate SVD of sparse matrix. Finally, for matrix completion we choose the SVT algorithm (an empirical comparison in Section IV.A shows its superiority to the IALM algorithm), and accelerate it with the rSVD-BKI algorithm and a novel subspace recycling technique. This results in a fast SVT algorithm with same accuracy and reliability as the original SVT algorithm. To demonstrate the efficiency of the proposed fast SVT algorithm, several color image inpainting and movie-rating estimation problems are tested. The results show that the proposed method consumes 15X and 8X less CPU time than the methods using \texttt{svds} and \texttt{lansvd} respectively, while outputting same-quality results.

For reproducibility, the codes and test data in this work will be shared on GitHub (\url{https://github.com/XuFengthucs/fSVT}).

\section{Preliminaries}
We assume that all matrices in this work are real valued, although the generalization
to complex-valued matrices is of no difficulty. In algorithm description, we  follow the Matlab convention for specifying row/column indices of a matrix.

\subsection{Singular Value Decomposition}

Singular value decomposition (SVD) is the most widely used matrix decomposition \cite{eckart1936,matrix2012}. Let $\mathbf{A}$ denote an $m \times n$ matrix. Its SVD is 
\begin{equation}
\mathbf{A}=\mathbf{U}\mathbf{\Sigma} \mathbf{V}^{\mathrm{T}} ,
\end{equation}
where orthogonal matrices $\mathbf{U}=[\mathbf{u}_1,\mathbf{u}_2,\cdots]$ and $\mathbf{V}=[\mathbf{v}_1,\mathbf{v}_2,\cdots]$ include the left and right singular vectors of $\mathbf{A}$, respectively. And, $\mathbf{\Sigma}$
is a diagonal matrix whose diagonal elements ($\sigma_1,\sigma_2,\cdots$) are the singular values of $\mathbf{A}$ in descending order. Suppose   $\mathbf{U}_k$ and $\mathbf{V}_k$ are the matrices with the first $k$ columns of  $\mathbf{U}$ and  $\mathbf{V}$, respectively, and $\mathbf{\Sigma}_k$ is a diagonal matrix containing the first $k$ singular values of $\mathbf{A}$. Then, we have the truncated SVD: 
\begin{equation}
\mathbf{A}\approx\mathbf{A}_k = \mathbf{U}_k\mathbf{\Sigma}_k\mathbf{V}_k^{\mathrm{T}} .
\end{equation}
It is  well known that this truncated SVD, i.e. $\mathbf{A}_k$, is the best rank-$k$ approximation of the matrix $\mathbf{A}$, in either spectral norm or Frobenius norm \cite{eckart1936}. 

To compute truncated SVD, a common choice is Matlab's built-in \texttt{svds} \cite{svds}. It is based on a Krylov subspace iterative method, and is especially efficient for handling sparse matrix. For a dense matrix $\mathbf{A}$, \texttt{svds} costs $O(mnk)$ flops for computing rank-$k$ truncated SVD. If $\mathbf{A}$ is sparse, the cost becomes $O(nnz(\mathbf{A})k)$ flops, where $nnz(\cdot)$ stands for the number of nonzeros of a matrix.  Another choice is PROPACK \cite{propack}, which is an efficient package in Matlab/Fortran for computing the dominant singular values/vectors of a large sparse matrix. The principal routine ``\texttt{lansvd}'' in PROPACK employs an  intricate Lanczos method to compute the singular values/vectors directly, instead of computing the eigenvalues/eigenvectors of an augmented matrix as in Matlab's built-in \texttt{svds}. Therefore, \texttt{lansvd} is usually several times faster than \texttt{svds}.

\subsection{Projection Based Randomized Algorithms}
The randomized algorithms have shown their advantages for solving the linear least squares problem and low-rank matrix approximation \cite{drineas2016randnla}. An idea is using random projection to identify the subspace capturing the dominant actions of matrix $\mathbf{A}$. This can be realized by multiplying $\mathbf{A}$ with a random matrix on its right side or left side, and then obtaining the subspace's orthonormal basis matrix $\mathbf{Q}$. With $\mathbf{Q}$, a low-rank approximation of $\mathbf{A}$ can be computed which further results in the approximate truncated SVD.
Because the dimension of the subspace is much smaller than that of  $range(\mathbf{A})$, this method facilitates the computation of near-optimal decompositions of $\mathbf{A}$. A basic randomized SVD (rSVD) algorithm is described as Algorithm 1  \cite{Halko2011Finding}.
\begin{algorithm}
    \caption{basic rSVD}
    \label{alg1}
    \begin{algorithmic}[1]
      \REQUIRE $\mathbf{A}\in\mathbb{R}^{m\times n}$, rank parameter $k$, power parameter $p$
      \ENSURE $\mathbf{U}\in\mathbb{R}^{m\times k}$, $\mathbf{S}\in\mathbb{R}^{k\times k}$, $\mathbf{V}\in\mathbb{R}^{n\times k}$
      \STATE $\mathbf{\Omega} = \mathrm{randn}(n, k+s)$
      \STATE $\mathbf{Q} = \mathrm{orth}(\mathbf{A}\mathbf{\Omega})$
      \FOR {$i=1, 2, \cdots, p$}
        \STATE $\mathbf{G} = \mathrm{orth}(\mathbf{A}^{\mathrm{T}}\mathbf{Q})$
        \STATE $\mathbf{Q} = \mathrm{orth}(\mathbf{A}\mathbf{G})$
      \ENDFOR
      \STATE $\mathbf{B}  = \mathbf{Q}^{\mathrm{T}}\mathbf{A}$
      \STATE $[\mathbf{U}, \mathbf{S}, \mathbf{V}] = \mathrm{svd}(\mathbf{B})$
      \STATE $\mathbf{U} = \mathbf{Q}\mathbf{U}$
      \STATE $\mathbf{U} = \mathbf{U}(:, 1:k), \mathbf{S} = \mathbf{S}(1:k, 1:k), \mathbf{V} = \mathbf{V}(:, 1:k)$.
    \end{algorithmic}
  \end{algorithm}

In Alg. 1, $\mathbf{\Omega}$ is a Gaussian i.i.d matrix. Other kinds of random matrix can replace $\mathbf{\Omega}$ to reduce the computational cost of $\mathbf{A\Omega}$, but they also bring some sacrifice on accuracy. With the subspace's orthogonal basis $\mathbf{Q}$, we have the approximation $\mathbf{A \approx QB}=\mathbf{QQ}^{\mathrm{T}}\mathbf{A}$. Then, performing the economic SVD on the $(k+s) \times n$ matrix $\mathbf{B}$ we obtain the approximate truncated SVD of $\mathbf{A}$.
To improve the accuracy of the QB approximation, a technique called power iteration (PI) scheme can be applied \cite{Halko2011Finding}, i.e. Steps 3$\sim$6. It is based on the fact that matrix $(\mathbf{AA}^\mathrm{T})^p \mathbf{A}$ has exactly the same singular vectors as $\mathbf{A}$, but its singular value decays more quickly. Thus, performing the randomized QB procedure on $(\mathbf{AA}^\mathrm{T})^p \mathbf{A}$ can achieve  better accuracy. The orthonormalization operation ``orth()'' is used to alleviate the round-off error in the floating-point computation.
More theoretical analysis can be found in \cite{Halko2011Finding}. 

The $s$ in Alg. 1 is an oversampling parameter, which enables $\mathbf{\Omega}$ with more than $k$ columns used for better accuracy. $s$ is  a small integer, 5 or 10. ``orth()'' is achieved  by a call to a packaged QR factorization (e.g., \texttt{qr(X, 0)} in Matlab).

The basic rSVD algorithm with the PI scheme has the following guarantee \cite{Halko2011Finding,musco2015}:
\begin{equation}
\|\mathbf{A}- \mathbf{Q}\mathbf{Q}^{\mathrm{T}}\mathbf{A} \| = \|\mathbf{A}- \mathbf{US}\mathbf{V}^{\mathrm{T}} \| \le (1+\epsilon ) \|\mathbf{A}- \mathbf{A}_k \| ,
\end{equation}
with a high probability ($\mathbf{A}_k$ is the best rank-$k$ approximation).

Another scheme called block Krylov-subspace iteration (BKI) can also be collaborated with the basic randomized QB procedure in Alg. 1. The resulted algorithm  satisfies (4) as well, and has better accuracy with same number of iteration ($p$ in Alg. 1). In \cite{musco2015}, it has been revealed that with the BKI scheme, the accuracy converges faster along with the iteration than using the PI scheme (Alg. 1). Specifically, the BKI scheme converges to the $(1+\varepsilon)$ low-rank approximation (4) in  $\tilde{O}(1/\sqrt{\varepsilon})$ iterations, while the PI scheme requires $\tilde{O}(1/\varepsilon)$ iterations.
This means that BKI based randomized SVD is more suitable for the scenario requiring higher accuracy. 

Some accelerating skills have been proposed to speed up the basic rSVD algorithm  \cite{Erichson_2017_ICCV,rsvdpack,alg971}, whose details will be addressed in the following section. However, they are developed individually and some of them just lack theoretic support. And, whether they are suitable for large sparse matrix is not well investigated.

\subsection{Matrix Completion Algorithms}
The matrix completion problem (1) is often relaxed to the problem minimizing the nuclear norm $\|\cdot\|_*$ of matrix:
\begin{equation}
\begin{aligned}
& \min_\mathbf{X} \|\mathbf{X}\|_* , ~ ~
s.t. ~ \mathcal{P}_{\Phi}(\mathbf{X})= \mathcal{P}_{\Phi}(\mathbf{M}),
\end{aligned}
\end{equation}
where $\mathcal{P}_{\Phi}(\cdot)$ is an orthogonal projector onto the span of matrices vanishing outside of set $\Phi$. 
The solution of (4) can be approached by an iterative process:
\begin{equation}
\begin{aligned}
\begin{cases}
\mathbf{X}^i= shrink( \mathbf{Y}^{i-1}, \tau),\\
\mathbf{Y}^i= \mathbf{Y}^{i-1}+ \delta \mathcal{P}_{\Phi}(\mathbf{M}-\mathbf{X}^i) .
\end{cases}
\end{aligned}
\end{equation}
Here, $\tau >0$, $\delta$ is a scalar step size, and $shrink(\mathbf{Y} ,\tau)$ is a function which applies a soft-thresholding rule at level $\tau$ to the singular values of matrix $\mathbf{Y}$. As the sequence $\{\mathbf{X}^i\}$ converges, one derives the singular value thresholding (SVT) algorithm for matrix completion (i.e. Algorithm 2) \cite{Cai2010}.
  \begin{algorithm}
    \caption{SVT}
    \label{alg2}
    \begin{algorithmic}[1]
      \REQUIRE  Sampled entries $\mathcal{P}_{\Phi}(\mathbf{M})$, tolerance parameter $\epsilon$
      \ENSURE $\mathbf{X}_{\mathrm{opt}}$
      \STATE $\mathbf{Y}^0 = c\delta\mathcal{P}_{\Phi}(\mathbf{M})$, ~$r_0 = 0$
      \FOR {$i=1, 2, \cdots, i_{\mathrm{max}}$}
        \STATE $k_i=r_{i-1}+1$
        \REPEAT
          \STATE $[\mathbf{U}^{i-1}, \mathbf{S}^{i-1}, \mathbf{V}^{i-1}] = \mathrm{svds}(\mathbf{Y}^{i-1}, k_i)$
          \STATE $k_i=k_i+l$
        \UNTIL $\mathbf{S}^{i-1}(k_i-l, k_i-l)\le \tau$
        \STATE $r_i = \mathrm{max}\{j:\mathbf{S}^{i-1}(j, j)>\tau\}$
        \STATE $\mathbf{X}^i=\sum^{r_i}_{j=1}(\mathbf{S}^{i-1}(j, j)-\tau)\mathbf{U}^{i-1}(:,j){(\mathbf{V}^{i-1}(:,j))}^{\mathrm{T}}$
        \STATE \textbf{if} $\left\|\mathcal{P}_{\Phi}(\mathbf{X}^i)\!-\!\mathcal{P}_{\Phi}(\mathbf{M})\right\|_{F}\!/\!\left\|\mathcal{P}_{\Phi}(\mathbf{M})\right\|_{F}\!<\!\epsilon$ \textbf{then break}
        \STATE $\mathbf{Y}^{i} = \mathbf{Y}^{i-1}+\delta(\mathcal{P}_{\Phi}(\mathbf{M})-\mathcal{P}_{\Phi}(\mathbf{X}^{i}))$
      \ENDFOR
      \STATE $\mathbf{X}_{\mathrm{opt}}=\mathbf{X}^{i}$
    \end{algorithmic}
  \end{algorithm}

In Alg. 2, ``svds($\mathbf{Y}, k$)''  computes rank-$k$ truncated SVD of $\mathbf{Y}$.
There are some internal parameters which follow the empirical settings in \cite{Cai2010}:
 $\tau=5n$, where $n$ is matrix column number,  $l=5$ and $c = \left\lceil \tau/(\delta\left\|\mathcal{P}_{\Phi}(\mathbf{M})\right\|_2)\right\rceil$. The value of $\delta$ affects the convergence rate, and one can slightly decrease it with the iteration. 

Due to space limit, we omit the details of  IALM algorithm  \cite{Lin2010}. In Section IV.A, with experiment we will show that the IALM is inferior to SVT algorithm for handling real data.

\section{Faster Randomized SVD for Sparse Matrix}
\subsection{The Ideas for Acceleration}
Because in each iteration of the SVT operation we need to compute truncated SVD of sparse matrix $\mathbf{Y}^{i-1}$, accelerating randomized SVD for sparse matrix is the focus. From Alg. 1, we see that Steps 2 and 7 occupy the majority of computing time if $\mathbf{A}$ is dense. However, for sparse matrix this is not true and optimizing other steps may bring substantial acceleration.

In existing work, some ideas were proposed to accelerate the basic rSVD algorithm.  In \cite{rsvdpack}, the idea of using eigendecomposition to compute the SVD in Step 8 of Alg. 1 was proposed. It was also pointed out that in the power iteration, orthonormalization after each matrix multiplication is not necessary. In \cite{alg971}, the power iteration was accelerated by replacing the QR factorization with LU factorization, and the Gaussian matrix is replaced with the random matrix with uniform distribution.
In \cite{Erichson_2017_ICCV}, the randomized SVD without power iteration was discussed for the dense matrix in image or video processing problem. It employs a variant of the basic rSVD algorithm, where the random matrix is multiplied to the left of $\mathbf{A}$. The algorithm is accelerated by using sparse random matrices and using eigendecomposition to obtain the orthonormal basis of the subspace.

Considering the situation for matrix completion, we decide only using the Gaussian matrix for $\mathbf{\Omega}$, because other choices are not suitable for sparse matrix (may cause $\mathbf{A\Omega}$ rank-deficient), and contribute little to the overall efficiency improvement. Other random matrix also degrades the accuracy of rSVD. The useful ideas for faster randomized SVD for sparse matrix are:
\begin{itemize}
\item use eigendecomposition for the economic SVD of $\mathbf{B}$;
\item perform orthonormalization after every other matrix-matrix multiplication in the power iteration;
\item perform LU factorization in the power iteration;
\item replace the orthonormal $\mathbf{Q}$ with the left singular vector matrix $\mathbf{U}$.
\end{itemize}


We first formulate the eigendecomposition based SVD computation as an eigSVD algorithm (described as Alg. 3), where ``eig()'' computes eigendecomposition. Its correctness is given as Lemma 1.
  \begin{algorithm}
    \caption{eigSVD}
    \label{alg3}
    \begin{algorithmic}[1]
      \REQUIRE $\mathbf{A}\in\mathbb{R}^{m\times n}$ ($m\ge n$)
      \ENSURE $\mathbf{U}\in\mathbb{R}^{m\times n}$, $\mathbf{S}\in\mathbb{R}^{n\times n}$, $\mathbf{V}\in\mathbb{R}^{n\times n}$
      \STATE $\mathbf{B} = \mathbf{A}^{\mathrm{T}}\mathbf{A}$
      \STATE $[\mathbf{V}, \mathbf{D}] = \mathrm{eig}(\mathbf{B})$
      \STATE $\mathbf{S} = \mathrm{sqrt}(\mathbf{D})$
      \STATE $\mathbf{U} = \mathbf{A}\mathbf{V}\mathbf{S}^{-1}$
    \end{algorithmic}
  \end{algorithm}
\begin{lemma}
The matrices $\mathbf{U, S, V}$ produced by Alg. 3 form the economic SVD of matrix $\mathbf{A}$.
\end{lemma}

\begin{proof}
Suppose $\mathbf{A}$ has SVD as (2). Since  $m\ge n$,
\begin{equation}
\mathbf{A}=\mathbf{U}(:,1:n)\mathbf{\tilde{\Sigma}}\mathbf{V}^{\mathrm{T}}, 
\end{equation}
where $\mathbf{\tilde{\Sigma}}$, a square diagonal matrix, is the first $n$ rows of $\mathbf{\Sigma}$. Eq. (7) is the economic SVD of $\mathbf{A}$.
Then, Step 1 computes
\begin{equation}
\mathbf{B} = \mathbf{A}^{\mathrm{T}}\mathbf{A}=\mathbf{V}\mathbf{\tilde{\Sigma}}^2 \mathbf{V}^{\mathrm{T}}. 
\end{equation}
The right-hand side is the eigendecomposition of $\mathbf{B}$. This means in Step 2, $\mathbf{D}=\mathbf{\tilde{\Sigma}}^2$ and $\mathbf{V}$ is the right singular vector matrix of $\mathbf{A}$. So, $\mathbf{S}$ in Step 3 equals $\mathbf{\tilde{\Sigma}}$, and lastly in Step 4 $\mathbf{U}= \mathbf{A}\mathbf{V}\mathbf{S}^{-1}= \mathbf{A}\mathbf{V}\mathbf{\tilde{\Sigma}}^{-1}= \mathbf{U}(:,1:n)$. The last equality is derived from (7). This proves the lemma.
\end{proof}

Notice that eigSVD is especially efficient if $m\gg n$, when $\mathbf{B}$ becomes a small $n\times n$ matrix. Besides, the singular values in $\mathbf{S}$ are in ascending order. Numerical issues can arise if matrix $\mathbf{A}$ has not full column rank. Though more efficient than standard SVD, eigSVD is only applicable to special situations.  

The idea that we can replace the orthonormal $\mathbf{Q}$ with the left singular matrix $\mathbf{U}$ can be explained with Lemma 2.
\begin{lemma}
In the basic rSVD algorithm, orthonormal matrix $\mathbf{Q}$ includes a set of orthonormal basis of subspace $range(\mathbf{A\Omega})$ or $range((\mathbf{AA}^\mathrm{T})^p\mathbf{A\Omega})$. No matter how $\mathbf{Q}$ is produced, the results of basic rSVD algorithm do not change.
\end{lemma}
\begin{proof}
The first statement is obviously correct by observing Alg. 1. The result of the basic rSVD algorithm is actually $\mathbf{QB}=\mathbf{QQ}^\mathrm{T}\mathbf{A}$, which further equals $\mathbf{USV}^\mathrm{T}$. Notice that $\mathbf{QQ}^\mathrm{T}$ is an orthogonal projector onto the subspace $range(\mathbf{Q})$, if $\mathbf{Q}$ is an orthonormal matrix. The orthogonal projector is uniquely determined by the subspace \cite{matrix2012}, here equals to $range(\mathbf{A\Omega})$ or $range((\mathbf{AA}^\mathrm{T})^p\mathbf{A\Omega})$. So, no matter how $\mathbf{Q}$ is produced, $\mathbf{QQ}^\mathrm{T}$ does not change, and the basic rSVD algorithm's results do not change.
\end{proof}

Both QR factorization and SVD of a same matrix produce the orthonormal basis of its range space (column space), in $\mathbf{Q}$ and $\mathbf{U}$ respectively. So, with Lemma 2, we can replace $\mathbf{Q}$ with $\mathbf{U}$ from SVD in the basic rSVD algorithm. 

Performing LU factorization is more efficient than QR factorization. It can be used while not affecting the correctness.
\begin{lemma}
In the basic rSVD algorithm, the ``orth()'' operation in the power iteration, except the last one, can be replaced by LU factorization. This does not affect the algorithm's accuracy in exact arithmetic.
\end{lemma}
\begin{proof}
Firstly, if the ``orth()'' is not performed, the power iteration produces $\mathbf{Q}$ including a set of basis of the subspace $range((\mathbf{AA}^\mathrm{T})^p\mathbf{A\Omega})$. As mentioned before, the ``orth()'' is just for alleviating the round-off error , and after using it $\mathbf{Q}$ still represents $range((\mathbf{AA}^\mathrm{T})^p\mathbf{A\Omega})$.
 
The pivoted LU factorization of a matrix $\mathbf{K}$ is:
\begin{equation}
\mathbf{PK=LU},
\end{equation}
where $\mathbf{P}$ is a permutation matrix, and $\mathbf{L}$ and $\mathbf{U}$ are lower triangular and upper triangular matrices respectively. Obviously, $\mathbf{K=(P}^\mathrm{T}\mathbf{L)U}$, where $\mathbf{P}^\mathrm{T}\mathbf{L}$ has the same column space as $\mathbf{K}$. So, replacing ``orth()'' with LU factorization (using $\mathbf{P}^\mathrm{T}\mathbf{L}$) also produces the basis of $range((\mathbf{AA}^\mathrm{T})^p\mathbf{A\Omega})$. Then, based on Lemma 2, we see  this does not affect the algorithm's results in exact arithmetic.
\end{proof}

Notice that the LU factor $\mathbf{P}^\mathrm{T}\mathbf{L}$ 
has scaled matrix entries with  linearly independent columns, since $\mathbf{L}$ is a lower-triangular matrix with unit diagonals and $\mathbf{P}$ just means row permutation. So, it also  alleviates the round-off error. Finally, the orthonormalization or LU factorization in the power iteration can be performed after every other matrix multiplication. It harms the accuracy little, but remarkably reduces runtime.


\subsection{Fast rSVD-PI Algorithm and rSVD-BKI Algorithm}
Based on the above discussion, we find out that the eigSVD procedure can be applied to the basic rSVD to produce both the economic SVD of $\mathbf{B}$ and the orthonormal $\mathbf{Q}$. Because in practice $k+s \ll m$ or $n$ and the matrices are not rank-deficient, using eigSVD induces no numerical issue. With these accelerating skills, we propose a fast rSVD-PI algorithm for sparse matrix (Alg. 4), where ``lu($\cdot$)'' denotes LU factorization function and its first output is ``$\mathbf{P}^\mathrm{T}\mathbf{L}$''.
  \begin{algorithm}
    \caption{rSVD-PI}
    \label{alg4}
    \begin{algorithmic}[1]
      \REQUIRE $\mathbf{A}\in\mathbb{R}^{m\times n}$, rank parameter $k$, power parameter $p$
      \ENSURE $\mathbf{U}\in\mathbb{R}^{m\times k}$, $\mathbf{S}\in\mathbb{R}^{k\times k}$, $\mathbf{V}\in\mathbb{R}^{k\times n}$
      \STATE $\mathbf{\Omega} = \mathrm{randn}(n, k+s)$
      \STATE $\mathbf{Q} = \mathbf{A}\mathbf{\Omega}$
      \FOR {$i=0, 1, 2, 3, \cdots, p$}
        \STATE \textbf{if} $i<p$ \textbf{then} [$\mathbf{Q}, \sim] = \mathrm{lu}(\mathbf{Q})$
        \STATE \textbf{else} $[\mathbf{Q}, \sim, \sim] = \mathrm{eigSVD}(\mathbf{Q})$ \textbf{break}
        \STATE $\mathbf{Q} = \mathbf{A}(\mathbf{A}^{\mathrm{T}}\mathbf{Q})$
      \ENDFOR
      \STATE $\mathbf{B}  = \mathbf{Q}^{\mathrm{T}}\mathbf{A}$
      \STATE $[\mathbf{V}, \mathbf{S}, \mathbf{U}] = \mathrm{eigSVD}(\mathbf{B}^{\mathrm{T}})$
      \STATE $ind = s+1:k+s$
      \STATE $\mathbf{U} = \mathbf{Q}\mathbf{U}(:, ind)$, $\mathbf{S} = \mathbf{S}(ind, ind), \mathbf{V} = \mathbf{V}(:, ind)$.
    \end{algorithmic}
  \end{algorithm}
\begin{theorem}
Alg. 4 is mathematically equivalent to the basic rSVD algorithm (Alg. 1).
\end{theorem}

\begin{proof}
One difference between Alg. 4 and Alg. 1 is in the power iteration (the ``for'' loop). Based on Lemma 1 we see that eigSVD accurately produces a set of orthonormal basis. And, based on Lemma 2 and 3, we see the power iteration in Alg. 4 is mathematically equivalent to that in Alg. 1. The other difference is the last three steps in Alg. 4. Its correctness is due to Lemma 1 and that the singular values produced by eigSVD is in the ascending order.
\end{proof}


For the scenario requiring higher accuracy, the BKI scheme \cite{musco2015} should be employed. Its main idea is to accumulate the subspaces generated in every iteration to form a larger subspace. Combining the accelerating skills we propose a fast BKI based rSVD algorithm (rSVD-BKI), i.e. Alg 5. 
Because the number of columns of $\mathbf{H}$ in Alg. 5 can be much larger than $\mathbf{Q}$'s in Alg. 4,  we use ``orth()'' instead of eigSVD to produce $\mathbf{Q}$ finally. Similarly, we have the following theorem.
\begin{theorem}
Alg. 5 is mathematically equivalent to the original BKI algorithm in \cite{musco2015}.
\end{theorem}
\begin{algorithm}
    \caption{rSVD-BKI}
    \label{alg5}
    \begin{algorithmic}[1]
      \REQUIRE $\mathbf{A}\in\mathbb{R}^{m\times n}$, rank parameter $k$, power parameter $p$
      \ENSURE $\mathbf{U}\in\mathbb{R}^{m\times k}$, $\mathbf{S}\in\mathbb{R}^{k\times k}$, $\mathbf{V}\in\mathbb{R}^{k\times n}$
      \STATE $\mathbf{\Omega} = \mathrm{randn}(n, k+s)$
      \STATE [$\mathbf{H}_{0}, \sim] = \mathrm{lu}(\mathbf{A}\mathbf{\Omega})$
      \FOR {$i=1, 2, 3, \cdots, p$}
        \STATE \textbf{if} $i<p$ \textbf{then} [$\mathbf{H}_{i}, \sim] = \mathrm{lu}(\mathbf{A}(\mathbf{A}^{\mathrm{T}}\mathbf{H}_{i-1}))$
      \ENDFOR
      \STATE $\mathbf{H} = [\mathbf{H}_0, \mathbf{H}_1, ..., \mathbf{H}_p]$
      \STATE $\mathbf{Q} = \mathrm{orth}(\mathbf{H})$
      \STATE $\mathbf{B}  = \mathbf{Q}^{\mathrm{T}}\mathbf{A}$
      \STATE $[\mathbf{V}, \mathbf{S}, \mathbf{U}] = \mathrm{eigSVD}(\mathbf{B}^{\mathrm{T}})$
      \STATE $ind = (k+s)(p+1)-k+1:(k+s)(p+1)$
      \STATE $\mathbf{U} = \mathbf{Q}\mathbf{U}(:, ind)$, $\mathbf{S} = \mathbf{S}(ind, ind), \mathbf{V} = \mathbf{V}(:, ind)$.
    \end{algorithmic}
  \end{algorithm}

Both Alg. 4 and Alg. 5 can accelerate the randomized SVD for sparse matrix. They do not reduce the major term in computational complexity, but have smaller scaling constants and reduce other terms. They also inherit the theoretical error bound of the original algorithms \cite{Halko2011Finding,musco2015}. Their accuracy and efficiency will be validated with experiments in Section V.A. Based on the rSVD-BKI, we will derive a fast SVT algorithm for matrix completion problems in the following section. 


\section{A Fast Matrix Completion Algorithm}

\subsection{The Choice of Algorithm}

To evaluate the quality of matrix completion, we consider the mean absolute error (MAE),
\begin{equation}
\mathrm{MAE} = \frac{\sum_{ij\in\Phi} |\mathbf{M}_{ij}-\tilde{\mathbf{M}}_{ij}|}{|\Phi|} ,
\end{equation}
where  $\mathbf{M}$ is the initial matrix, $\tilde{\mathbf{M}}$ is the recovered matrix, and $|\Phi|$ is the number of samples. MAE can be measured on the sampled matrix entries, or the whole matrix entries if the initial matrix is known. Before developing a faster matrix completion algorithm, we compare the  SVT and IALM  algorithms for recovering a $2,048\times 2,048$ color image from  20\%  pixels in it (i.e. Case 2 in Section V.II). The MAE curves along the iteration steps produced with SVT and IALM algorithms are shown in Fig. 1. From it we see that SVT achieves much better accuracy than
\begin{figure}[h]         
  \centering
\includegraphics[width=3.6in, height=1.8in]{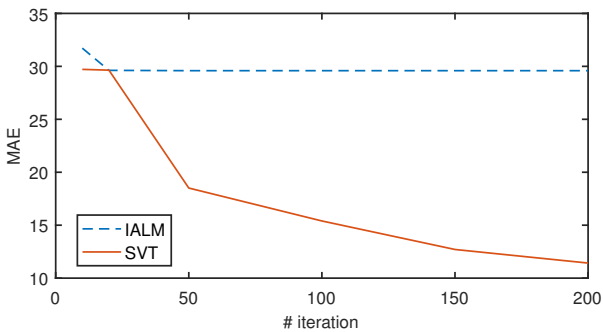}
  \caption{The accuracy convergence of SVT and IALM algorithms.}
  \label {Fig1}
\end{figure}
 IALM, though the latter converges faster. A probable reason is that the IALM algorithm works well on some low-rank data, instead of the real data.


In this work, we focus on the acceleration of the SVT algorithm.
A basic idea is replacing the truncated SVD in SVT algorithm with the fast rSVD algorithms in last section. However, with the iterations in SVT algorithm advancing, the rank parameter $k_i$ becomes very large. Calculating so many singular values/vectors accurately is not easy. Firstly, the rSVD-BKI algorithm is  preferable, which will be demonstrated with experiment in Section V.A. Secondly, with a large power $p$, its runtime advantage over \texttt{svds} or \texttt{lansvd} may lose, so that other accelerating technique is needed. 


\subsection{Subspace Recycling}
The SVT algorithm uses an iterative procedure to build up the low rank approximation, where truncated SVD is repeatedly carried out on $\mathbf{Y}^i$. According to Step 11 in Alg. 2,
\begin{equation}
\begin{aligned}
\left\|\mathbf{Y}^{i}\!-\!\mathbf{Y}^{i-1}\right\|_F\!=\!\delta\left\|\mathcal{P}_{\Phi}(\mathbf{M}\!-\!\mathbf{X}^{i})\right\|_F 
\!\le\!\delta \left\|\mathbf{M}\!-\!\mathbf{X}^{i}\right\|_{F}
\end{aligned}
\end{equation}
Because $\left\|\mathbf{M}-\mathbf{X}^{i}\right\|_{F} \to 0$ when iteration index $i$ becomes large enough (see Theorem 4.2 in \cite{Cai2010}), Eq. (11) means $\mathbf{Y}^{i}$ is very close to the $\mathbf{Y}^{i-1}$. So are the truncated SVD results of $\mathbf{Y}^{i}$ and $\mathbf{Y}^{i-1}$. The idea is to reuse the subspace of $\mathbf{Y}^{i-1}$ calculated in previous iteration step to speed up the SVD computation of $\mathbf{Y}^{i}$. This should be applied when $i$ is large enough. Two recycling strategies are: 
\begin{itemize}
\item Reuse the orthogonal basis $\mathbf{Q}$ in the rSVD-BKI for $\mathbf{Y}^{i-1}$, and then start from Step 8 in the rSVD-BKI algorithm for computing SVD of $\mathbf{Y}^{i}$.
\item Reuse the left singular vectors $\mathbf{U}^{i-1}$ in last iteration to calculate SVD of $\mathbf{Y}^{i}$, with the following steps. 
\begin{algorithm}
    \label{alg7}
    \begin{algorithmic}[1]
      \STATE $\mathbf{B} = {\mathbf{U}^{i-1}}^{\mathrm{T}}\mathbf{Y}^{i}$
      \STATE $[\mathbf{V}^{i}, \mathbf{S}^{i}, \mathbf{U}^{i}] = \mathrm{eigSVD}(\mathbf{B}^{\mathrm{T}})$
      \STATE $\mathbf{U}^{i} = \mathbf{U}^{i-1}\mathbf{U}^{i}$
    \end{algorithmic}
  \end{algorithm}
\end{itemize}

The second strategy costs less time, because the size of $\mathbf{U}^{i-1}$ is $m\times k$ while the size of $\mathbf{Q}$ is $m\times (p+1)(k+s)$. However, it is less accurate than the first one. So, the second strategy is suitable for the situation where the error reduces rapidly in the iterative process of SVT algorithm, e.g. the image inpainting problem.


\subsection{Fast SVT Algorithm}
Based on the proposed techniques, we obtain a fast SVT algorithm described as Alg. 6.
\begin{algorithm}
    \caption{fast SVT}
    \label{alg6}
    \begin{algorithmic}[1]
      \REQUIRE Sampled entries $\mathcal{P}_{\Phi}(\mathbf{M})$, tolerance $\epsilon$
      \ENSURE $\mathbf{X}_{\mathrm{opt}}$
      \STATE $\mathbf{Y}^0 = c\delta\mathcal{P}_{\Phi}(\mathbf{M})$, ~$r_0=0$, ~$q=0$, ~$p=3$
      \FOR {$i=1, 2, \cdots, i_{\mathrm{max}}$}
        \STATE $k_i=r_{i-1}+1$, adjust the value of $p$
        \REPEAT
          \STATE \textbf{if} $i<i_{\mathrm{reuse}}$ \textbf{or}  $q == q_{\mathrm{reuse}}$ \textbf{then} 
          \STATE ~~~$[\mathbf{U}^{i-1}, \mathbf{S}^{i-1}, \mathbf{V}^{i-1}] =$ rSVD-BKI$(\mathbf{Y}^{i-1}, k_i, p)$
          \STATE ~~~$q=0$
          \STATE \textbf{else}
          \STATE ~~~reuse $\mathbf{Q}$ or $\mathbf{U}$ in last execution of rSVD-BKI \\ ~~~algorithm and compute $\mathbf{U}^{i-1}, \mathbf{S}^{i-1}, \mathbf{V}^{i-1}$
          \STATE ~~~$q = q+1$
          \STATE \textbf{end if}
          \STATE $k_i=k_i+l$
        \UNTIL $\mathbf{S}^{i-1}(k_i-l,k_i-l)\le \tau$
        \STATE $r_i = \mathrm{max}\{j:\mathbf{S}^{i-1}(j,j)>\tau\}$
        \STATE $\mathbf{X}^i=\sum^{r_i}_{j=1}(\mathbf{S}^{i-1}(j,j)-\tau)\mathbf{U}^{i-1}(:,j){(\mathbf{V}^{i-1}(:,j))}^{\mathrm{T}}$
        \STATE \textbf{if} $\left\|\mathcal{P}_{\Phi}(\mathbf{X}^i)\!-\!\mathcal{P}_{\Phi}(\mathbf{M})\right\|_{F}\!/\!\left\|\mathcal{P}_{\Phi}(\mathbf{M})\right\|_{F}\!<\!\epsilon$ \textbf{then break}
        \STATE $\mathbf{Y}^{i} = \mathcal{P}_{\Omega}(\mathbf{Y}^{i-1})+\delta(\mathcal{P}_{\Phi}(\mathbf{M})-\mathcal{P}_{\Phi}(\mathbf{X}^{i}))$
      \ENDFOR
      \STATE $\mathbf{X}_{\mathrm{opt}}=\mathbf{X}^{i}$
    \end{algorithmic}
  \end{algorithm}
$i_{reuse}$ represents the minimum iteration to execute subspace recycling, and $q_{reuse}$ represents the maximum times of subspace recycling with one subspace. To guarantee the accuracy of randomized SVD, the power parameter $p$ should increase with the iteration because the rank $k_i$ of $\mathbf{Y}^i$ increases. Our strategy is increasing $p$ by 1 once the relative error in Step 16 increases.
This ensures a gradual decrease of error. And, if the error continuously decreases for 10 times, we reduce $p$ by 1. This prevents overstating $p$. 
Other parameters follow the settings for Alg. 2 (see Section II.B).

Here we would like to explain the convergence of the proposed fast SVT algorithm.  As proved in \cite{musco2015}, the BKI based randomized SVD is able to attain any high accuracy if $p$ is large enough. So is our rSVD-BKI algorithm.  In the fast SVT algorithm (Alg. 6), the $k$-truncated SVD is computed and $k$ increases with the iterations. We initially set a $p$ value which enables the rSVD-BKI algorithm attains same accuracy as \texttt{svds} for computing a few leading singular values/vectors. With the iteration advancing a mechanism gradually increasing $p$ value is applied, such that rSVD-BKI can accurately compute more leading singular values/vectors. As a result, this accurate SVD computation guarantees that Alg. 6 behaves the same as the original SVT algorithm using \texttt{svds}. On the other hand, Theorem 4.2 in \cite{Cai2010} proves the convergence of the original SVT algorithm. So, the convergence of our Alg. 6 is also guaranteed.

Notice that the subspace recycling technique is inspired by the theoretic analysis of (11). We have devised two recycling strategies and restrict their usage with parameters $i_{reuse}$ and $q_{reuse}$. They, to some extent, ensure that the accuracy in the fast SVT algorithm will not degrade after incorporating the subspace recycling. This has been validated with extensive experiments, some of which are given in Section V.II and V.III.


\section{Experimental Results}
All experiments are carried out on a computer with Intel Xeon CPU @2.00 GHz and 128 GB RAM. The algorithms have been implemented in Matlab 2016a. \texttt{svds} in Matlab and  \texttt{lansvd} in PROPACK \cite{propack} are used in Alg. 2, respectively. The resulted algorithms are compared with the proposed fast SVT algorithm (Alg. 6). The CPU time of different algorithms are compared, which is irrespective of the number of threads used in different SVT implementations.

The test cases for matrix completion are color images and movie-rating matrices from the MovieLens  datasets \cite{movielens}. Below, we first evaluate the proposed fast rSVD algorithms for sparse matrix and then validate the fast SVT algorithm.

\subsection{Validation of Fast rSVD Algorithms}
In this subsection, we first compare our rSVD-PI algorithm (Alg. 4) with the basic rSVD, cSVD (using \texttt{randn} as the random matrix) \cite{Erichson_2017_ICCV}, pcafast \cite{alg971}, rSVDpack \cite{rsvdpack} algorithms. We consider a sparse matrix in size 45,115 $\times$ 45,115 obtained from the  MovieLens dataset. The matrix has 97 nonzeros per row on average and is denoted by Matrix 1. Then, we randomly set some nonzero elements to zero to get two sparser matrices: Matrix 2 and 3 with 24 and 9 nonzeros per row on average, respectively. Setting rank $k=100$, we performed the truncated SVD with different algorithms. The results are listed in Table I. 
\begin{table*}[h]
\setlength{\abovecaptionskip}{0.05 cm}
 \caption{The Computational Results of Different Randomized SVD Algorithms ($k=100$). The Unit of CPU Time Is Second}
  \label{tab:table2}
  \centering
\renewcommand{\multirowsetup}{\centering}
  \begin{tabular}{ccccccccccc} 
  \toprule
 \multicolumn{2}{c}{Setting} &
  \multicolumn{3}{c}{Matrix 1} & \multicolumn{3}{c}{Matrix 2} &\multicolumn{3}{c}{Matrix 3}\\
  \cmidrule(r){1-2}
  \cmidrule(r){3-5} \cmidrule(r){6-8} \cmidrule(r){9-11}
   Algorithm & $p$ & $t_{cpu}$ & Error & Sp. & $t_{cpu}$ & Error & Sp. & $t_{cpu}$ & Error & Sp. \\
\midrule
 basic rSVD (Alg.1) & $0$ & 6.19 & 0.8166 & * & 5.10 & 0.9341 & * & 4.98 & 0.9506 & *\\
 cSVD\cite{Erichson_2017_ICCV} & $0$ & 2.76 & 0.8166 & 2.2 & 1.74 & 0.9352 & 2.9 & 1.51 & 0.9508 & 3.3\\
 pcafast\cite{alg971} & $0$ & 5.92 & 0.8188 & 1.0 & 5.04 & 0.9338 & 1.0 & 4.66 & 0.9506 & 1.1\\
 rSVDpack\cite{rsvdpack} & $0$ & 2.59 & 0.8186 & 2.4 & 1.67 & 0.9355 & 3.1 & 1.67 & 0.9506 & 3.0 \\
 rSVD-PI (Alg.4) & $0$ & \textbf{2.10} & 0.8156 & 3.0 & \textbf{1.10} & 0.9342 & 4.8 & \textbf{0.84} & 0.9502 & \textbf{6.0}\\
 \midrule
 basic rSVD (Alg.1) & $4$ & 18.7 & 0.7305 & *  & 13.2 & 0.8614 & * & 12.1 & 0.8804 & *\\
 cSVD\cite{Erichson_2017_ICCV} & $4$ & 14.9 & 0.7305 & 1.3 & 9.70 & 0.8614 & 1.4 & 8.51 & 0.8805 & 1.4\\
 pcafast\cite{alg971} & $4$ & 12.9 & 0.7305 & 1.5 & 8.36 & 0.8615 & 1.6 & 6.69 & 0.8805 & 1.8\\
 rSVDpack\cite{rsvdpack} & $4$ & 11.7 & 0.7305 & 1.6 & 6.32 & 0.8617 & 2.1 & 5.40 & 0.8804 & 2.2\\
 rSVD-PI (Alg.4) & $4$ & \textbf{8.32} & 0.7305 & 2.2 & \textbf{3.18} & 0.8615 & 4.2 & \textbf{2.02} & 0.8804 & \textbf{6.0}\\
\bottomrule 
 \end{tabular}
\end{table*}
Error there is the approximation error $\left\|\mathbf{A}-\tilde{\mathbf{A}}_k\right\|_F/\left\|\mathbf{A}\right\|_F$, where $\tilde{\mathbf{A}}_k$ denotes the computed rank-$k$ approximation.

From the table we see that the proposed rSVD-PI algorithm has same accuracy as the basic rSVD algorithm, but is from 2.2X to 6.0X faster (Sp. in Table I denotes the speedup ratio to the basic rSVD). And, for a sparser matrix the speedup ratio increases. If the power iteration is not imposed ($p=0$), cSVD and rSVDpack perform well, with at most 3.3X and 3.0X speedup respectively. When $p=4$, the speedup ratios of these methods decrease. 
However, rSVDpack is better, due to the improvement of power iteration. 
pcafast also shows moderate speedup because it  replaces QR with LU factorizaiton. These results verify the efficiency of our rSVD-PI algorithm for handling sparse matrix. It has up to \textbf{6.0X} speedup over the basic rSVD algorithm, and is several times faster than other state-of-the-art rSVD approaches. 

Considering the scenario needing high accuracy, we compare rSVD-PI and rSVD-BKI algorithms with various matrices. Different values of power parameter $p$ are tested and the results of \texttt{svds} are also given as the baseline. The results for Matrix 1 (setting $k=100$) are listed in Table II. From it, we see that rSVD-BKI can reach the accuracy of \texttt{svds} in shorter runtime and a smaller $p=4$. However, rSVD-PI cannot attain the accuracy of \texttt{svds} even when $p$ is as large as 15. The experimental results show that rSVD-BKI achieves better accuracy than rSVD-PI in shorter CPU time, with much smaller $p$. This verifies that the rSVD-BKI algorithm (Alg. 5) is more efficient than rSVD-PI for high-precision computation. 

As we have tested, to ensure the accuracy of SVD in the SVT iterations, the power $p$ can increase to several tens while using rSVD-PI algorithm or similar randomized algorithms. This largely increase the runtime and makes rSVD-PI and those algorithms in Table I no competitive advantage over the standard SVD methods. So, we can only use rSVD-BKI in the following matrix completion experiments.  

\begin{table}[h]
\setlength{\abovecaptionskip}{0.05 cm}
 \caption{The Comparison of rSVD-PI and rSVD-BKI Algorithms}
  \label{tab:table2}
  \centering
\renewcommand{\multirowsetup}{\centering}
  \begin{tabular}{cccc} 
  \toprule
  {Algorithm}  & {$t_{cpu}$ (s)} & Error & Sp.\\
\midrule
\texttt{svds}  & 75.0 & 0.7289 & * \\
rSVD-PI (Alg.4), $p=2$ & 5.20 & 0.7345  & 14\\
rSVD-PI (Alg.4), $p=15$ & 26.4 & 0.7290 & 2.8 \\
rSVD-BKI (Alg.5), $p=4$ & 22.0 & 0.7289 & 3.4 \\
\bottomrule 
 \end{tabular}
\end{table}

\subsection{Image Inpainting}
In this subsection, we test the matrix completion algorithms with a landscape color image. It includes $2,048\times 2,048$ pixels, and we stack the three color channels of it to get a matrix in size of $6,144\times 2,048$. Then, we randomly sample 10\% and 20\% pixels to construct Case 1 and Case 2 for image inpainting, respectively. The error of image inpainting is measured with the MAE on all image pixels. 

For the two cases, $\epsilon$ in the SVT algorithms is set 0.052 and 0.047 respectively. They correspond to the situation where the error of matrix completion does not decrease any more. The parameters for subspace recycling are $q_{reuse} = 10$, $i_{reuse}=100$. And, we use the second recycling strategy reusing $\mathbf{U}$ matrix. Our fast SVT algorithm is compared with the SVT algorithm using \texttt{svds} and \texttt{lansvd}, see Table III.
\begin{table}[h]
\setlength{\abovecaptionskip}{0.05 cm}
  \caption{The Results of Image Inpainting (Unit of CPU Time Is Second)}
  \label{tab:table3}
  \centering
\renewcommand{\multirowsetup}{\centering}
  \begin{tabular}{@{}c@{~}c@{~}c@{~}c@{~}c@{~}c@{~}c@{~}c@{~}c@{}} 
  \toprule
  \multirow{2}{*}{Test case}  &
  \multicolumn{3}{c}{SVT (Alg.2)} & \multicolumn{3}{c}{fast SVT (Alg.6)} &\multirow{2}{*}{Sp1} & \multirow{2}{*}{Sp2}\\
  \cmidrule(r){2-4} \cmidrule(r){5-7}
   & $t_{\texttt{svds}}$ & $t_{\texttt{lansvd}}$ & MAE & $t_{w/o}$ & $t_{w/}$ & MAE &  &  \\
\midrule
Case 1 & 10,674 & 6,295 & 11.69 & 1,812 & 944 & 11.69 & 11.3 & 6.7\\
Case 2 & 19,358 & 10,008 & 8.854 & 3,254 & \textbf{1,279} & 8.854 & \textbf{15.1} & \textbf{7.8} \\ 
\bottomrule 
 \end{tabular}
\end{table}

In Table III, $t_{\texttt{svds}}$ and $t_{\texttt{lansvd}}$ denote the CPU time of the SVT algorithms (Alg. 2) using \texttt{svds} and \texttt{lansvd} for truncated SVD, respectively. $t_{w/o}$ and $t_{w/}$ denote the CPU time of our fast SVT Algorithm (Alg. 6) without and with subspace recycling, respectively. Sp1 and Sp2 are the ratios of $t_\texttt{svds}$ and $t_\texttt{landsvd}$ to the CPU time of our algorithm (with subspace recycling). We can see that the proposed algorithm is up to \textbf{15.1X} and \textbf{7.8X} faster than the SVT algorithms using \texttt{svds} and \texttt{lansvd}, respectively. 
Its memory cost is 512 MB, slightly larger than 420 MB used by the SVD algorithm with \texttt{lansvd}. 
All algorithms present the same accuracy (same MAE value), with same iteration numbers (400 for Case 1 and 700 for Case 2). The rank of the result matrix is 109 for Case 1 or 102 for Case 2. Comparing $t_{w/o}$ and $t_{w/}$ we see that the subspace recycling technique brings about 2X speedup, while not degrading the accuracy.

The recovered images from Case 1 
are shown in Fig. 2, along with the original image. It reveals that our Alg. 6 produces same quality as the original SVT algorithm.
\begin{figure}[h]         
  \centering
  \setlength{\abovecaptionskip}{0.02 cm}
  \subfigure[Initial image] {\includegraphics[width=1.6in, height=1.6in]{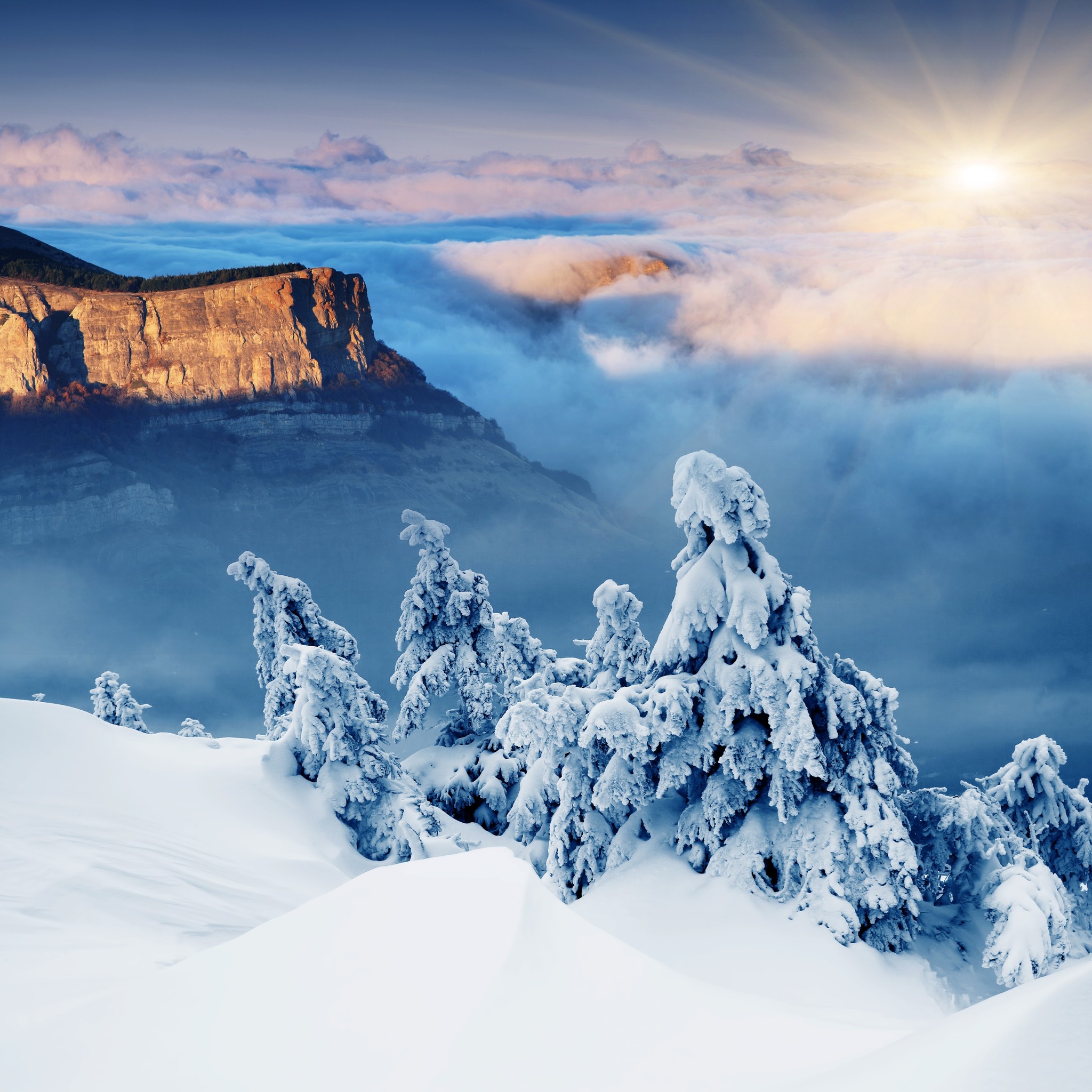}}
  \subfigure[Sampled 10\% pixels] {\includegraphics[width=1.6in, height=1.6in]{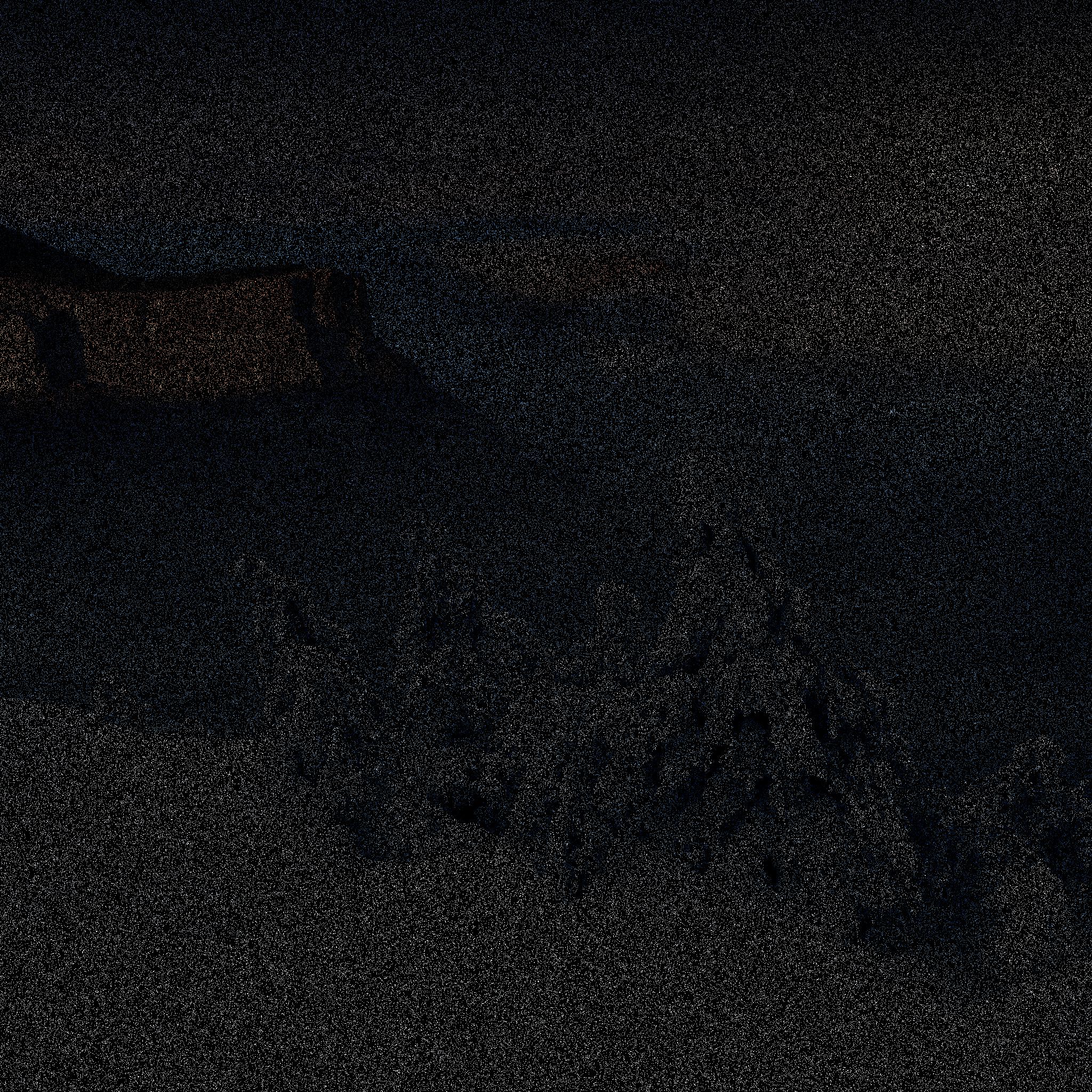}}
  \subfigure[Recovered with Alg. 2] {\includegraphics[width=1.6in, height=1.6in]{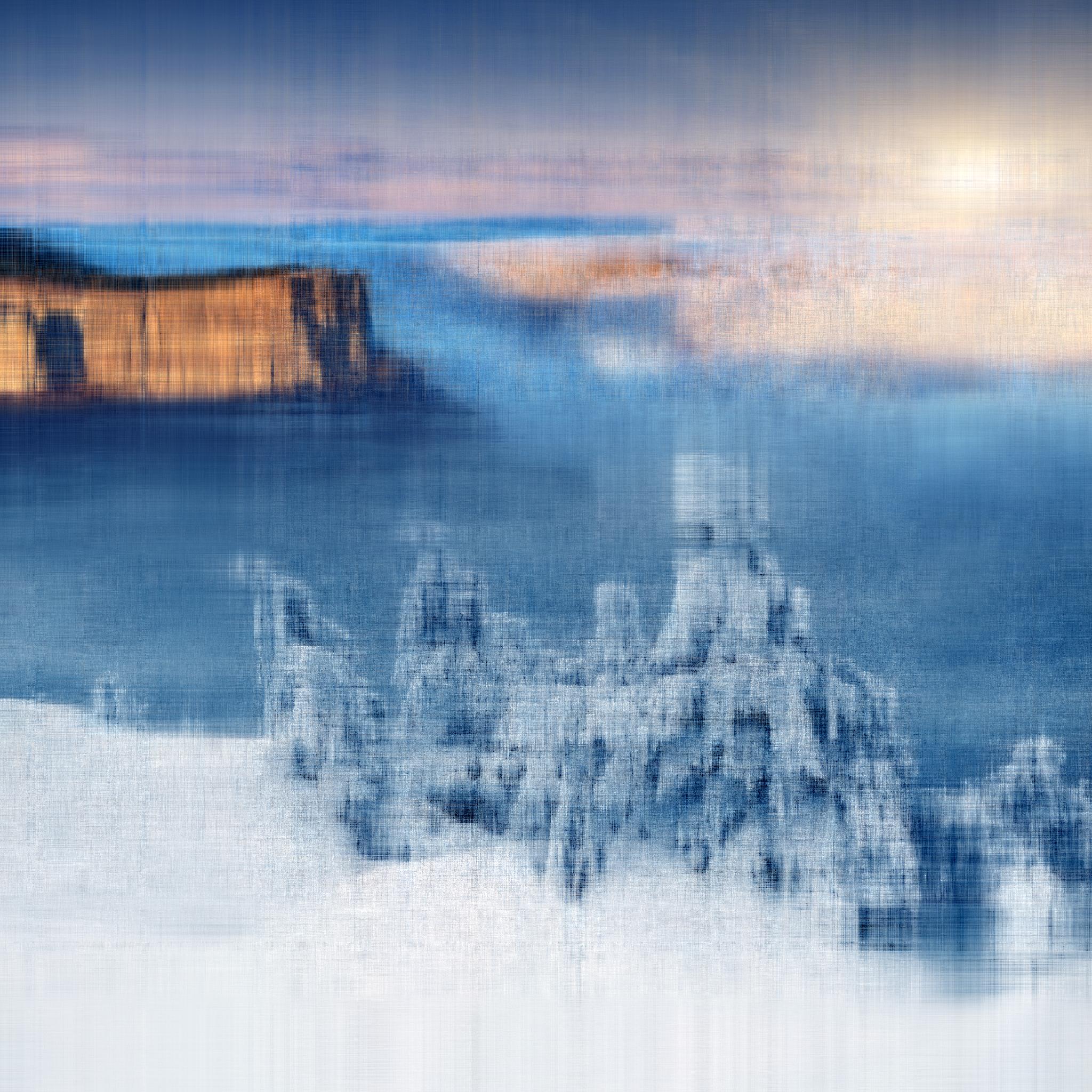}}
  \subfigure[Recovered with Alg. 6] {\includegraphics[width=1.6in, height=1.6in]{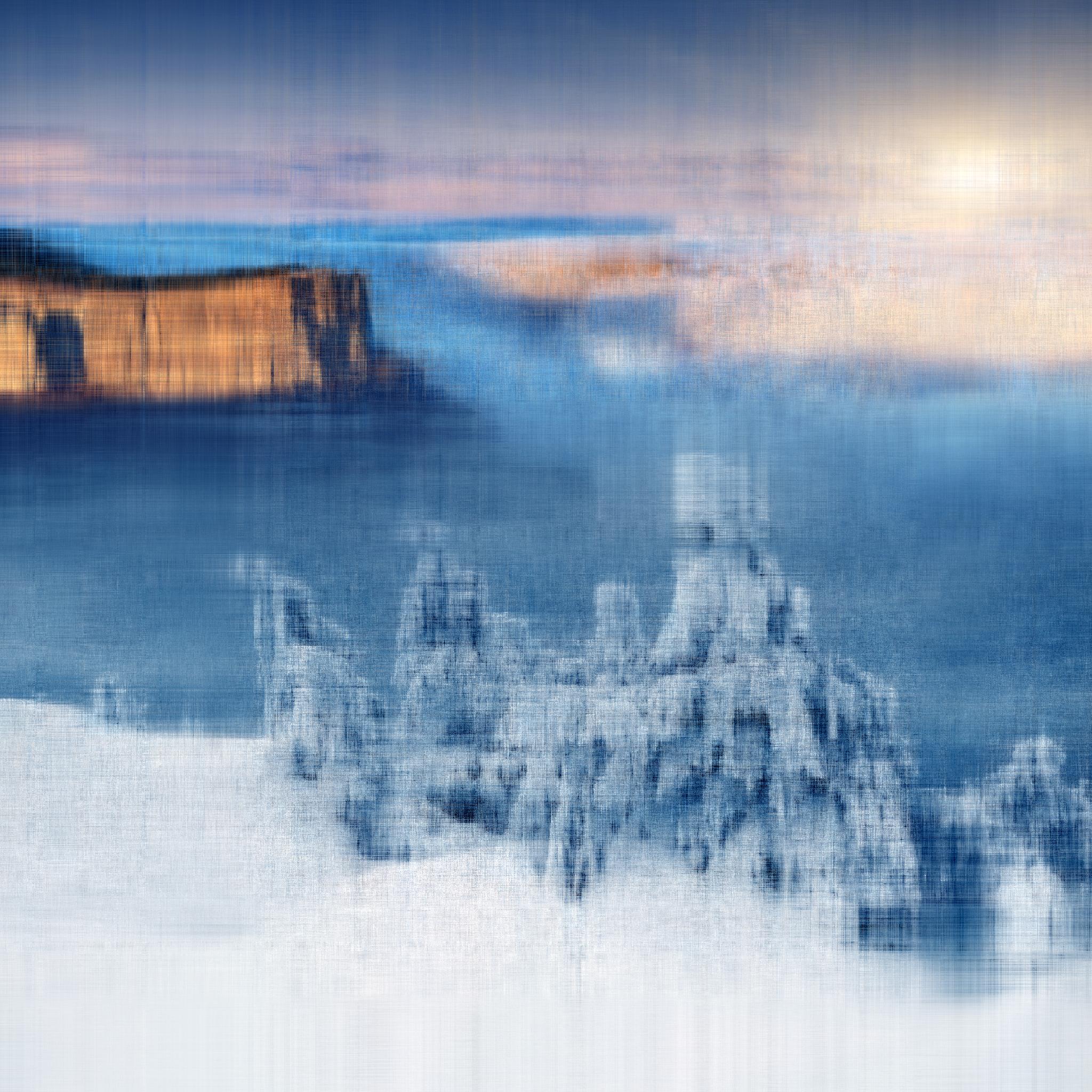}}
  \caption{The initial image and recovered images from 10\% pixels.}
  \label {Fig2}
\end{figure}

\subsection{Rating Matrix Completion}
The rating matrix includes users' ratings  to movies, ranged from 0.5 to 5. For each dataset we keep a portion of ratings to be the training set. With them we recover the whole rating matrix and then use the remaining ratings to evaluate the accuracy of the matrix completion.
In this experiment, $q_{reuse} = 10$, $i_{reuse}=50$, 
and the first subspace recycling strategy is used because it delivers better accuracy.

We first test a smaller dataset, including 10,000,054 ratings from 71,567 users judging 10,677 movies. We randomly sample 80\% and 90\% ratings as the training sets to obtain  Case 3 and Case 4, respectively. The $\epsilon$ in the SVT algorithms is 0.16 and 0.19 for the both cases respectively, corresponding to the situation where the error of matrix completion does not decrease any more. The experimental results are in Table IV. 
\begin{table}[h]
\setlength{\abovecaptionskip}{0.05 cm}
  \caption{The Results of Rating Matrix Completion for a Smaller Dataset}
  \label{tab:table6}
  \centering
\renewcommand{\multirowsetup}{\centering}
  \begin{tabular}{@{}c@{~}c@{~}c@{~}c@{~}c@{~}c@{~}c@{~}c@{}} 
  \toprule
  \multirow{2}{*}{Test case}  &
  \multicolumn{3}{c}{SVT (Alg.2)} & \multicolumn{2}{c}{fast SVT (Alg.6)} &\multirow{2}{*}{Sp1} & \multirow{2}{*}{Sp2}\\
  \cmidrule(r){2-4} \cmidrule(r){5-6}
   & $t_{\texttt{svds}}$ & $t_{\texttt{lansvd}}$ & MAE & ~ ~$t_{w/}$ & MAE &  &  \\
\midrule
Case 3 & 75,297 & 48,290 & 0.6498 & ~ ~15,133 & 0.6501 & 5.0 & 3.2 \\ 
Case 4 & 19,813 & 12,509 & 0.6458 & ~ ~\textbf{3,771} & 0.6460 & \textbf{5.3} & \textbf{3.3}\\
\bottomrule 
 \end{tabular}
\end{table}

According to Table IV, we see that the fast SVT algorithm has same accuracy as the original SVT algorithm. Here, MAE is measured on the remaining ratings. With the proposed techniques, the fast SVT algorithm is up to 5.3X and 3.3X faster than the methods using $t_\texttt{svds}$ and $t_\texttt{lansvd}$, respectively.
From MAE we see that the error of rating estimation is on average 0.65, which is moderate.

Then, we  test a larger dataset which includes 20,000,263 ratings from 138,493 users to 26,744 movies. It derives Case 5 and Case 6 by sampling  80\% and 90\% known ratings. The computational results are listed in Table V. They confirm the accuracy of the proposed algorithm again, and show its speedup up to 4.8X. It should be pointed out that the number of iterations to achieve the best quality in SVT algorithms are 362 for Case 5 and 293 for Case 6, which are larger than 208 for Case 3 and 153 for Case 4. But the ranks of the result matrix are 58 for Case 5 and 45 for Case 6 which are much smaller than 239 for Case 3 and 138 for Case 4. This explains why the CPU time for handling the larger dataset is less than that for handling the smaller dataset.
\begin{table}[h]
\setlength{\abovecaptionskip}{0.05 cm}
  \caption{The Results of Rating Matrix Completion for a Larger Dataset}
  \label{tab:table7}
  \centering
\renewcommand{\multirowsetup}{\centering}
    \begin{tabular}{@{}c@{~}c@{~}c@{~}c@{~}c@{~}c@{~}c@{~}c@{}} 
  \toprule
  \multirow{2}{*}{Test case}  &
  \multicolumn{3}{c}{SVT (Alg.2)} & \multicolumn{2}{c}{fast SVT (Alg.6)} &\multirow{2}{*}{Sp1} & \multirow{2}{*}{Sp2}\\
  \cmidrule(r){2-4} \cmidrule(r){5-6}
   & $t_{\texttt{svds}}$ & $t_{\texttt{lansvd}}$ & MAE & ~ ~$t_{w/}$ & MAE &  &  \\
\midrule
Case 5 & 30,213 & 15,213 & 0.6676 & ~ ~6,582 & 0.6676 & 4.6 & 2.3\\ 
Case 6 & 19,951 & 9,785 & 0.6685 & ~ ~4,180 & 0.6685 & 4.8 & 2.3\\
\bottomrule 
 \end{tabular}
\end{table}

\section{Conclusions}
In this paper, we have presented two contributions. Firstly, a fast randomized SVD technique is proposed for sparse matrix. It results in two fast rSVD algorithms: rSVD-PI and rSVD-BKI. The former is faster than all existing approaches and up to 6X faster than the basic rSVD algorithm, while the latter is even better for problem requiring higher accuracy. 
 Then, utilizing the rSVD-BKI, we propose a fast SVT algorithm for matrix completion. It also includes a new subspace recycling technique and is applied to the problems of image inpainting and rating matrix completion. The experiments with real data show that the proposed algorithm brings up to 15X speedup without loss of accuracy.

In the future, we will explore the application of this fast matrix completion algorithm to more AI problems.


\bibliographystyle{IEEEtran}
\bibliography{IEEEfull, ictai\_sub4}

\end{document}